\documentclass{article}

\usepackage[final,nonatbib]{neurips_2020}

\usepackage[utf8]{inputenc} 
\usepackage[T1]{fontenc}    
\usepackage{url}            
\usepackage{booktabs}       
\usepackage{amsfonts}       
\usepackage{nicefrac}       
\usepackage{microtype}      

\usepackage{amsmath} 
\usepackage{amssymb}
\usepackage{color} 
\usepackage{graphicx}
\usepackage{algorithm}
\usepackage{algpseudocode}
\usepackage{wrapfig}            
\usepackage{amsthm}
\usepackage{thmtools}
\usepackage{multicol}
\usepackage{commands}

\title{Direct Policy Gradients: Direct Optimization of Policies in Discrete Action Spaces}

\author{
  Guy Lorberbom \\
  Technion
  \And
  Chris J. Maddison \\
  DeepMind
  \And
  Nicolas Heess \\
  DeepMind
  \And
  Tamir Hazan \\
  Technion
  \And
  Daniel Tarlow \\
  Google Research, Brain Team
}

\begin{document}
\maketitle

\begin{abstract}
Direct optimization \cite{mcallester2010direct, song2016training} is an appealing framework that replaces integration with optimization of a random objective for approximating gradients in models with discrete random variables \cite{lorberbom2018directopt}. 
A$^\star$ sampling \cite{maddison2014astar} is a framework for optimizing such random objectives over large spaces. 
We show how to combine these techniques to yield a reinforcement learning algorithm that approximates a policy gradient by finding trajectories that optimize a random objective.  
We call the resulting algorithms \emph{direct policy gradient} (DirPG) algorithms.
A main benefit of DirPG algorithms is that they allow the insertion of domain knowledge in the form of upper bounds on return-to-go at training time, like is used in heuristic search, while still directly computing a policy gradient.
We further analyze their properties, showing there are cases where DirPG has an exponentially larger probability of sampling informative gradients compared to REINFORCE.
We also show that there is a built-in variance reduction technique and that a parameter that was previously viewed as a numerical approximation can be interpreted as controlling risk sensitivity.
Empirically, we evaluate the effect of key degrees of freedom and show that the algorithm performs well in illustrative domains compared to baselines.
\end{abstract}

\section{Introduction}

Many problems in machine learning reduce to learning a probability distribution (or policy) over sequences of discrete actions so as to maximize a downstream utility function.
Examples include generating text sequences to maximize a task-specific metric like BLEU and generating action sequences in reinforcement learning (RL) to maximize expected return.
A main challenge is that evaluating the objective requires integrating over all possible sequences, which is intractable, and thus approximations like REINFORCE are needed \cite{williams1992simple} to learn these policies.

A line of work has emerged in recent years that allows replacing integration and sampling with optimization of noisy objective functions \cite{papandreou2011perturb,hazan2012partition,tarlow2012randoms,maddison2014astar,lorberbom2018directopt}.
While this does not immediately remove the intractability of the integration problem, casting the problem in terms of optimization gives access to a different toolbox of ideas, which can provide new perspectives and methods for these hard problems.
For example, Maddison et al. provide a new way of leveraging bounds from convex duality for use in sampling from continuous probability distributions \cite{maddison2014astar}.
Our aim in this work is the analog for reinforcement learning: we will replace the integral that is typically approximated by REINFORCE with an alternative that requires only optimization over a noisy objective function. 
The benefit is that this opens up techniques from heuristic search for use in reinforcement learning (e.g., variants of A$^\star$ search) and provides an opportunity to express domain knowledge, all while retaining the conceptual simplicity that comes from optimizing a standard expected return objective function.

The resulting algorithm is quite different from standard approaches to computing a policy gradient, but it estimates the same quantity up to one finite difference approximation.
We provide a comprehensive analysis of the new algorithm from both theoretical and empirical perspectives.
In total, this work provides a new perspective on computing a policy gradient and expands the toolbox of techniques and domain knowledge that can be used to tackle this fundamental problem.

\section{Background}
\label{sec:background}

\paragraph{Reinforcement learning.} We consider a standard problem of RL, in which an agent interacts with a Markov Decision Process (MDP) for a finite number of steps\footnote{Technically, everything in the paper works with an unbounded numbers of steps as long as trajectories terminate with probability 1, but we assume a maximum number of steps to simplify some parts of the exposition.} and attempts to maximize its accumulated reward. At any time $t \geq 0$ the environment is in state $s_t \in \states$ in the given state space $\states$; there is a fixed initial state $s_0 \in \states$. At each time $t$ the agent interacts with the environment by taking an action $a_t$ from a finite set of actions $a_t \in \actions$ according to a policy parameterized by $\theta \in \R^d$, 
    $\modelp{a_t \mid s_t}$.
The environment follows a transition distribution $p(r_t, s_{t+1} \mid s_t, a_t)$ over rewards $r_t$ and next states $s_{t+1}$ given previous state $s_t$ and action $a_t$. The agent interacts with the environment in this way for $T > 0$ steps generating a sequence of states $\statetrajectory = (s_1, \ldots, s_T)$, actions $\actiontrajectory = (a_0, \ldots, a_{T-1})$, and rewards $\rewardtrajectory = (r_0, \ldots, r_{T-1})$. This corresponds to the following generative model,
\begin{equation}
    \label{eq:mdp_genmodel} 
    \begin{aligned}
        a_t &\sim \modelp{\cdot \mid s_t} \text{ for } t \in \{0, \ldots, T-1\}\\
        r_t, s_{t+1} &\sim p(\cdot, \cdot \mid a_t, s_t) \text{ for } t \in \{0, \ldots, T-1\}\\
    \end{aligned}
\end{equation}
given $s_0 \in \states$. Taken together this defines 
the following joint distribution,
\begin{align}
    \label{eq:trajectorydist} p_{\theta}(\trajectorytriple) = \prod_{t=0}^{T-1} \modelp{a_t \mid s_t} p(r_t, s_{t+1} \mid s_t, a_t).
\end{align}
The sum of rewards  $r_t$ over an interaction is called the return, and the goal of the agent is to maximize the expected return over its policy parameters,
$\max_{\theta \in \R^d} \, \expect_{\trajectorytriple \sim p_{\theta}}\left[ \sum_{t=0}^{T-1} r_t \right]$. 

\paragraph{Policy gradients.} Policy gradient algorithms are a family of methods for optimizing expected return by estimating gradients. A common variant is REINFORCE \cite{williams1992simple}, which samples a trajectory $\trajectorytriple \sim p_{\theta}$, computes the return $R = \sum_{t=0}^{T-1} r_t$, and then approximates the gradient as $R \cdot \nabla_\theta \log p_{\theta}(\trajectorytriple)$.

\paragraph{Gumbel-max reparameterizations.} A random variable $G\sim \Gumbel(m)$ is Gumbel-distributed with location $m$ if $p(G \le g) = \exp(-\exp(-g + m))$.
The Gumbel-max trick is a way of casting sampling from a softmax as an $\argmax$ computation by using the fact that if $G(i)$ are drawn i.i.d.~as $\Gumbel(m_i)$, then $i^* = \argmax_{i} G(i) \sim \exp(m_i) / \sum_{i'} \exp (m_{i'})$. Moreover, $G^* = \max_{i} G(i) \sim \Gumbel(\log \sum_{i'} \exp m_{i'})$ and $i^*$ and $G^*$ are independent random variables. See \cite{gumbel1954statistical,maddison2014astar,maddison2017gumbel}.

\paragraph{Direct optimization. }
Direct optimization \cite{mcallester2010direct, song2016training, lorberbom2018directopt} approximates gradients of a loss function over discrete configurations that are computed as the $\argmax$ of a (possibly noisy) underlying potential function. 
Following \cite{lorberbom2018directopt} and letting $\trajectory$ be a discrete variable, $f_\theta$ be a scoring function, $\epsilon$ be an auxiliary variable, $G(\trajectory) \sim \Gumbel(0)$ be independent Gumbel noise, and $r$ be a negative loss function, the method is based on a \emph{direct} objective $D_\theta(\trajectory, G, \epsilon) = f_\theta(\trajectory) + G(\trajectory) + \epsilon \cdot r_{\trajectory}$.
The main result is that 
$\nabla_\theta \expect_{G} \left[r_{\argmax_{\trajectory} f_\theta(\trajectory) + G(\trajectory)}  \right] = \lim_{\epsilon \rightarrow 0} \frac{1}{\epsilon} \expect_{G} \left[ \nabla_\theta f_\theta(\argmax_{\trajectory} D_\theta(\trajectory, G, \epsilon)) - \nabla_\theta f_\theta(\argmax_{\trajectory} D_\theta(\trajectory, G, 0)) \right]$.



\section{Basic Algorithm, Motivating Example, and Summary of Results}
\label{sec:motivation}

To motivate the approach as simply as possible, we first present a minimal version of our new \emph{\ouralg~(\ouralgshort)} algorithm and an example where it has an exponentially larger probability of sampling an informative gradient compared to REINFORCE.
In later sections we will handle the full complexity of RL, justify correctness, and describe how to efficiently compute the needed quantities.

DirPG utilizes optimization to find an informative gradient that improves the reward of its policy. In contrast, REINFORCE samples from its current policy. This inherent difference can allow DirPG to find a policy gradient more efficiently than REINFORCE. In the following we formalize this difference by considering a simple environment where rewards $r_{\trajectory}$ are a function of action sequences $\trajectory \in \actions^T$ for a large action space $|\actions|^T$. Further suppose that we are in a sparse reward regime such that $r_{\othertrajectory} = m > 0$ for one trajectory $\othertrajectory$ and $r_{\trajectory} = 0$ for all others. 
The REINFORCE gradient is $r_{\trajectory} \nabla \log \modelp{\trajectory}$ where $\trajectory \sim \pi$.
Since $r_{\trajectory}$ is zero for most $\trajectory$,  $k$ samples from a uniform policy $\pitheta$ (like would arise at the start of learning) will result in a nonzero gradient with probability roughly $\frac{k}{|\actions|^T}$.

In this setting \ouralgshort~can be described as follows. Let $G(\trajectory) \sim \Gumbel(0)$ be independent Gumbel noise for each trajectory $\trajectory$ and $\epsilon$ a hyperparameter.
There are two trajectories of interest: 
\begin{align}
    \topt &= \argmax_{\trajectory} \left[ \log \modelp{\trajectory} + G(\trajectory) \right] \label{eq:simple_topt} \\
    \tdirect &= \argmax_{\trajectory} \left[ \log \modelp{\trajectory} + G(\trajectory) + \epsilon \cdot r_{\trajectory} \right].  \label{eq:simple_tdir}
\end{align}
The direct policy gradient is defined as 
\begin{align}
    \nabla_{\theta} \expect_{\trajectory \sim \pitheta} r_{\trajectory}
    & \approx \frac{1}{\epsilon}  \left[
    \nabla_\theta \log \modelp{\tdirect} - \nabla_\theta \log \modelp{\topt} \right]. \label{eq:simplified_dirpg}
\end{align}
A key benefit of \ouralgshort~is that domain knowledge may be inserted to guide a search for $\tdirect$. 
Suppose we have a powerful search heuristic that leads directly to the optimum. 
Then $\tdirect$ can be computed at the same cost as a single sample, and the total cost of an update requires only two samples (for $\topt$ and $\tdirect$), hence its computational complexity is equivalent to REINFORCE with $k=2$.
DirPG computes an informative (non-zero) gradient iff $\tdirect = \othertrajectory$ and $\topt \not = \othertrajectory$. The probability of $\topt \not = \othertrajectory$ is large ($1 - \frac{1}{|\actions|^T}$), so this mainly comes down to whether $\tdirect = \othertrajectory$, which is equivalent to the event
\begin{align}
    \log \modelp{\othertrajectory}  + G(\othertrajectory) + \epsilon \cdot m > \max_{\trajectory \not = \othertrajectory} \left[ \log \modelp{\trajectory} + G(\trajectory) \right].
\end{align}
When $\pitheta$ is uniform, this simplifies to 
$\epsilon \cdot m + G(\othertrajectory)  > \max_{\trajectory  \not = \othertrajectory} G(\trajectory)$,
which has the form of sampling $\othertrajectory$ via Gumbel-max. 
The RHS has distribution $\Gumbel(\log (|\actions|^T - 1))$ and thus the probability of sampling $\tdirect = \othertrajectory$ is 
$\frac{\exp(\epsilon \cdot m)}{\exp(\epsilon \cdot m) + \exp ( \log (|\actions|^T - 1) )}$. 
If $\epsilon$ scales logarithmically with $|A|^T$, then \ouralgshort~has an exponentially higher chance than REINFORCE to sample an informative gradient in this example.

This example motivates \ouralgshort~and also raises a number of questions. 
In the remainder, we provide a comprehensive analysis of the new algorithm.
Since the algorithm has many facets, we prioritize the following, leaving developing finely-tuned variants that outperform state of the art to future work:

{\bf Full complexity of RL. } We show how to handle general stochastic environments (\secref{sec:dirpg}).
\\[5pt] {\bf Correctness. } We show that \ouralgshort~computes a policy gradient up to a one-dimensional finite difference approximation that leads to the appearance of $\epsilon$ (\secref{sec:dirpg}).
\\[5pt] {\bf Utilizing existing heuristics. } We assumed above that a perfect heuristic enables computing $\tdirect$ at the cost of a single rollout. This elides an important detail, which is that the heuristic
must not only guide search to maximize return, it must also consider the $\log \pitheta + G$ terms in \eqref{eq:simple_tdir}. By extending A$^\star$ sampling, we show how to convert a heuristic over returns to a heuristic for computing $\tdirect$ (\secref{sec:topdown}).
\\[5pt] {\bf Approximate optimization. } With imperfect heuristics, exactly computing $\tdirect$ can be intractable. We define a notion of \emph{improvement} over $\topt$ and prove (in a restricted setting) that approximate optimization of $\tdirect$ still leads to learning an optimal policy (\appref{app:approximate_optimization}). 
\\[5pt] {\bf Epsilon. } Previous work on direct optimization \cite{mcallester2010direct} recognized that $\epsilon$ could be positive (``towards good'') or negative (``away from bad'') but did not provide a precise analysis of its impact. We provide a novel interpretation, deriving the objective optimized under different choices of $\epsilon$ and show there is a precise connection to risk-aware RL (\appref{app:risk}).
\\[5pt] {\bf Variance Reduction. } We show that \ouralgshort~``comes with its own variance reduction,'' by providing an interpretation of the $\nabla_\theta \log \modelp{\topt}$ term in \eqref{eq:simplified_dirpg} as a control variate (\appref{app:variance_reduction}).
\\[5pt] {\bf Empirical analysis. } We study all of the above in a set of carefully designed experiments that illustrate how to leverage the large literature on heuristic-guided search in specific domains, and the effect of key parameters like $\epsilon$ and the approximation of $\tdirect$ (\secref{sec:experiments}).

\section{Direct Policy Gradient}
\label{sec:dirpg}

We start by formalizing the full \ouralgshort~algorithm in a general stochastic RL environment. 
Note that there are two places where stochasticity enters into \eqref{eq:trajectorydist}:
via the agent's policy in the $\modelp{a_t \mid s_t}$ terms and via the environment
in the $p(r_t, s_{t+1} \mid s_t, a_t)$ terms.
Given this factorization, we can separately reparameterize them.
Once this is done, the direct optimization approach follows straightforwardly.
A key requirement of the learning update is that we can explore multiple trajectories for a given realization of environment noise, so the method requires a simulator in order to compute a gradient. However, the result of learning is a standard policy that can be sampled from without any search or simulator, so, e.g., it would be feasible to use in sim-to-real settings.

\paragraph{Reparameterization. } The learning rule for \ouralgshort~is based on search over trajectories and thus requires a simulator for computing a gradient.
Beyond that, we do not want to restrict the environments, so we consider a very general reparameterization, which is simply that there is some source of randomness $\staterewardtree$ that does not depend on $\trajectory$ such that there is a deterministic function mapping $\staterewardtree$ and a sequence of $(s_0, a_0, \ldots, s_t, a_t)$ to the next $r_t, s_{t+1}$ pair.
This implies, for example, that if $\staterewardtree$ is held fixed and an agent performs the same sequence of actions, then the same environment transitions and rewards will be produced. We denote the state (reward) resulting from a sequence of actions $\trajectory$ as $s_{\trajectory}$ ($r_{\trajectory})$. When clear from context, we omit the explicit dependence on $\staterewardtree$ for brevity.

Now it becomes straightforward to define a \emph{per-trajectory} Gumbel-max reparameterization. Let the total log probability that a policy assigns to a sequence of actions be
\begin{equation}
    \modelP{\actiontrajectory \mid \staterewardtree} = \prod_{t=0}^{T-1} \modelp{a_t \mid s_{(a_0 \ldots a_{t-1})}},
    \label{eq:trajectory_logprob}
\end{equation}
and let $\gumbels(\trajectory) \sim \Gumbel(0)$ for each trajectory $\trajectory$. 
This yields a trajectory-level Gumbel-max trick:
\begin{eqnarray}
    \Gtheta(\trajectory; \gumbels, \staterewardtree) &=& \log \modelP{\trajectory \mid \staterewardtree} + \gumbels(\trajectory) \\ 
\trajectory^* &=& 
            \argmax_{\trajectory}
            \Gtheta(\trajectory; \gumbels, \staterewardtree).
            \label{eq:G}
\end{eqnarray}

$\Gtheta$ are distributed as Gumbels with shifted locations and $\trajectory^*$ is a sample from \eqref{eq:trajectory_logprob}.

We emphasize that the reparameterization is equivalent to the standard RL formulation.
Specifically, let $P(\staterewardtree)$ be the distribution over $\staterewardtree$ resulting from different realizations of environment stochasticity and let the return of a trajectory $\trajectory$ be $R(\trajectory, \staterewardtree) = \sum_{t=0}^{T-1} r_{(a_0, \ldots, a_{t-1})}$. Then 
\begin{eqnarray}
    \expect_{\trajectorytriple \sim p_{\theta}}\left[ \sum_{t=0}^{T-1} r_t \right] 
    \;=\; \expect_{\staterewardtree \sim P}\left[ \expect_{\trajectory \sim \modelP{\cdot \mid \staterewardtree}} \left[ R(\trajectory, \staterewardtree)  \right]\right] 
    \;=\; \mathbb{E}_{\staterewardtree \sim P, \gumbels} \left[ R(\trajectory^*, \staterewardtree) \right].
\end{eqnarray}


\paragraph{Direct Policy Gradient.}
The above reparameterizations allow defining the general \ouralgshort~algorithm and showing its correctness. Define \emph{direct objective} $D_\theta$ and \emph{prediction generating function} $f$:
\begin{align}
\dobj(\trajectory; \gumbels, \staterewardtree, \epsilon) & = \Gtheta(\trajectory; \gumbels, \staterewardtree) + \epsilon R(\trajectory, \staterewardtree), \label{eq:direct_obj} \\
 f(\theta, \epsilon)   & = \mathbb{E}_{\staterewardtree \sim P, \gumbels} \left[ 
            \max_{\trajectory}
            \left\{
            \dobj(\trajectory; \gumbels, \staterewardtree,  \epsilon)
            \right\}
            \right],\\
 \tstar{\epsilon} & = \argmax_{\trajectory} \dobj(\trajectory;\gumbels, \staterewardtree, \epsilon). \label{eq:tstar}
\end{align}
When clear from context, we drop the explicit dependence on noise terms $\staterewardtree$ and $\gumbels$ for brevity. Differentiating $f$ with respect to $\epsilon$ and $\theta$ in either order and evaluating at $\epsilon=0$ yields the same value because $f$ is smooth \cite{lorberbom2018directopt} (or see \cite{song2016training} for an alternative proof):
\begin{equation}
\frac{\partial}{\partial \theta_i}
        \mathbb{E}\left[ 
        R(\trajectory^*(0), \staterewardtree)
            \right]
= \left.\frac{\partial^2 f(\theta, \epsilon)}{\partial \theta_i \partial \epsilon}  \right|_{\epsilon = 0} 
=
\left. \frac{\partial^2 f(\theta_i, \epsilon)}{\partial \epsilon \partial \theta_i} \right|_{\epsilon = 0} 
=
 \left.\frac{\partial}{\partial \epsilon} \mathbb{E}\left[ 
            \frac{\partial}{\partial \theta_i} \log \modelP{\tstar{\epsilon} \mid \staterewardtree}
            \right]\right|_{\epsilon = 0}.
            \label{eq:differentiate_second_way}    
\end{equation}
%

A finite-difference approximation in $\epsilon$ of the RHS of \eqref{eq:differentiate_second_way} yields the direct policy gradient (\ouralgshort):
\begin{align}
\label{eq:dlpg}
\nabla_\theta \expect_{\trajectorytriple \sim p_{\theta}}\left[ \sum_{t=0}^{T-1} r_t \right] & \approx
\frac{1}{\epsilon} \mathbb{E}_{\staterewardtree \sim P, \gumbels} \left[
    \nabla_\theta \log \modelP{\tstar{\epsilon} \mid \staterewardtree}
    - \nabla_\theta \log \modelP{\tstar{0} \mid \staterewardtree}\right]. 
\end{align}
Following terminology of \cite{mcallester2010direct} we name 
$\topt = \tstar{0}$ as the optimum in Eq.~\ref{eq:G},
and $\tdirect = \tstar{\epsilon}$ as the trajectory that defines the update direction. 
Because the LHS of \eqref{eq:differentiate_second_way} is the gradient of the expected return, \ouralgshort{} approaches the standard policy gradient as $\epsilon \to 0$.

Intuitively, \eqref{eq:tstar} reduces to \eqref{eq:G} when $\epsilon=0$, so $\topt$ is a trajectory sampled from the current policy.
$\tdirect$ is a trajectory that is close to a sample from the current policy but that has higher or lower return, where the strength and direction of this pull comes from the magnitude and sign of $\epsilon$. The gradient increases the probability of the better trajectory and decreases the worse.

\begin{wrapfigure}{r}{0.5\textwidth}
\vspace{-12pt}
\begin{minipage}[t]{.5\textwidth}
\begin{algorithm}[H]
{\footnotesize
    \centering
    \caption{{\footnotesize Direct Policy Gradient (General Form)}\label{alg:dpg}}
    \begin{algorithmic}[1]
        \State $\staterewardtree \sim P(\staterewardtree)$
        \State $\gumbels(\trajectory) \sim \Gumbel(0) \hbox{ for all $\trajectory$}$
        \State $\trajectoryGenerator = \TrajectoryGenerator(\staterewardtree, \gumbels, \epsilon)$
        \State $\topt, d_{opt} \gets \tdirect, d_{dir} \gets \trajectoryGenerator.\next()$
        \While{budget not exceeded}
            \State $\trajectory_{cur}, d_{cur} \gets \trajectoryGenerator.\next()$  
            \If {$d_{cur} > d_{dir}$}
                \State $\tdirect, d_{dir} \gets \trajectory_{cur}, d_{cur}$
                \If{terminate on first improvement}
                    \State break
                \EndIf 
            \EndIf
        \EndWhile
        \State {\bf return} $\frac{1}{\epsilon} \nabla_\theta \left[ \log \modelp{\tdirect \mid \staterewardtree} - \log \modelp{\topt \mid \staterewardtree} \right]$
    \end{algorithmic}
}
\end{algorithm}
\vspace{-10pt}
\end{minipage}
\end{wrapfigure}

\paragraph{Algorithms.}
The general form of algorithms we consider is given in Algorithm~\ref{alg:dpg}.
The basis is a $\TrajectoryGenerator$ (see \secref{sec:topdown}) that produces a stream of pairs of trajectories $\trajectory$ and associated direct objectives $\dobj(\trajectory; \epsilon)$.
The first step of Algorithm~\ref{alg:dpg} is to find $\topt$ and $d_{opt}=\dobj(\topt; 0)$ and initialize $\tdirect=\topt, d_{dir}=d_{opt}$.
The algorithms in \secref{sec:topdown} naturally produce $\topt$ and $d_{opt}$ as the first result, so we assume that behavior.
The algorithm then applies heuristic search to find a trajectory $\tdirect$ with direct objective $d_{dir}$ better than $d_{opt}$ (lines 5-13).
If no improvement is found before a budget is exceeded, 
then $\topt$ is equal to $\tdirect$ and the result of line 14 is a zero gradient.
Given enough budget and no early termination, the algorithm exactly implements \eqref{eq:dlpg}.
One variant is to terminate the search upon finding any improvement (line 9).
This automatically adapts the search budget as training progresses. 
At first it is easy to improve over $\topt$ (a sample from a random policy), 
but more search is needed after training for longer.
In \appref{app:approximate_optimization}, we prove that this variant
still learns an optimal policy (in a restricted setting).

\section{Generating trajectories using $A^\star$ sampling}
\label{sec:topdown}

 $A^\star$ sampling provides a starting point for computing $\topt$ and $\tdirect$, but
it is inefficient in its use of environment interactions. 
Here, we develop a new variant tailored to the RL setting that 
uses a lazier sampling strategy that minimizes the number of environment interactions. Despite $\topt$ being an argmax over $|\actions|^T$ trajectories, the algorithm produces an exact solution in $T$ steps.
Computing $\tdirect$ is more challenging, but \ouralgshort~can leverage heuristics to guide the search, and it benefits relative to REINFORCE by actively searching for an informative gradient.

\begin{wrapfigure}{r}{0.5\textwidth}
    \vspace{-14pt}
    \centering
    \begin{tabular}{cc}
         \includegraphics[width=.21\columnwidth]{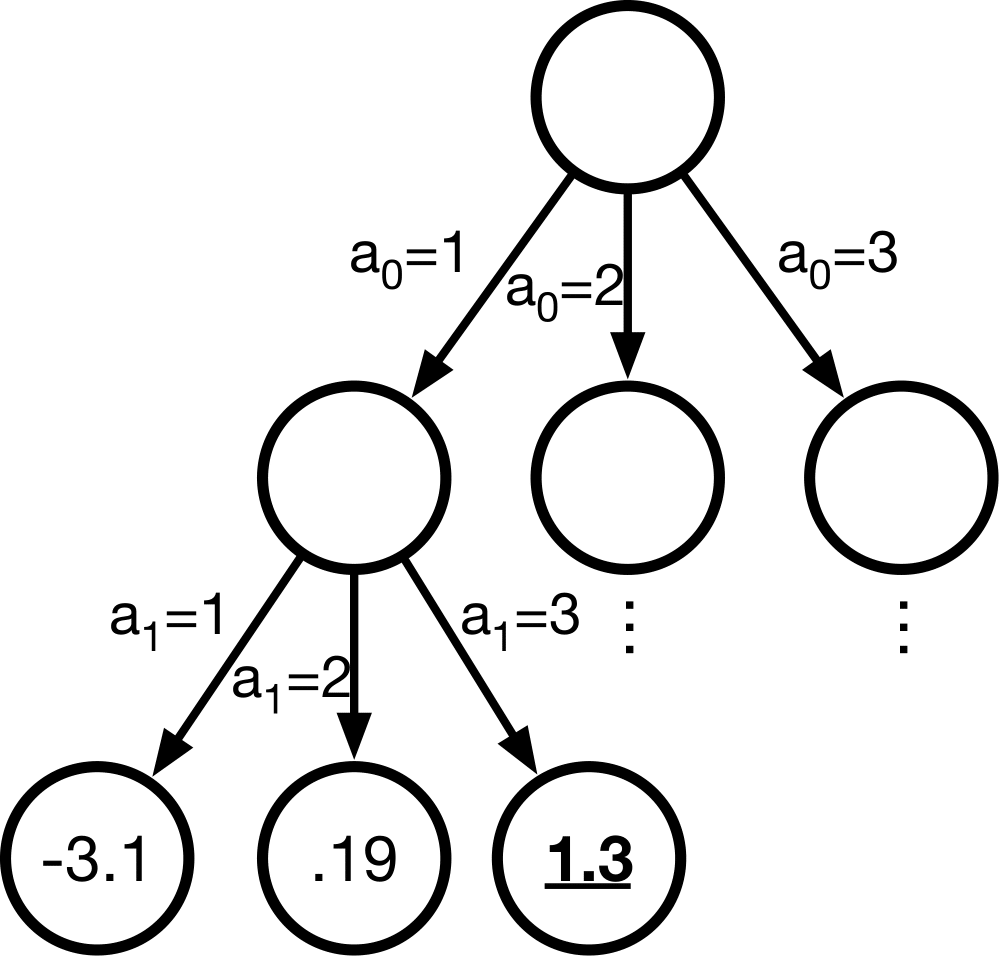}
         & \includegraphics[width=.18\columnwidth]{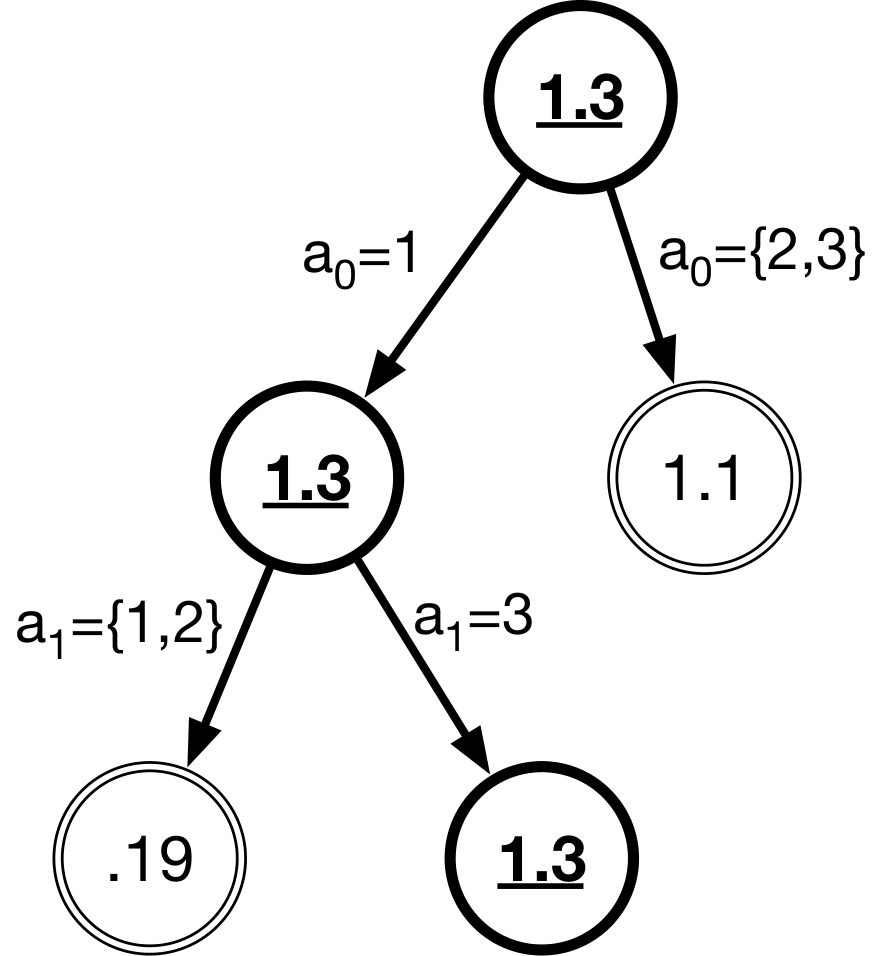} \\
         (a) {\footnotesize Gumbels for trajectories} &          \hspace{-10pt}
         (b) {\footnotesize Gumbels for regions}
    \end{tabular}
    \vspace{-2pt}
    \caption{
    {\footnotesize
    Example search tree and associated values. 
    {\bf (a)} Gumbel values $\Gtheta(\trajectory; \gumbels, \staterewardtree)$ associated with each trajectory $\trajectory$.
    The trajectory with maximum value (underlined) is $\topt$.
    {\bf (b)} State of the search tree after sampling $\topt$. 
    Nodes on the queue are drawn with double outline.
    }}
    \label{fig:state_reward_and_search_trees}
    \vspace{-8pt}
\end{wrapfigure}
\paragraph{Search Space.}
The search over $\staterewardtree$ for $\topt$ and $\tdirect$ is structured into a search tree over sets of action sequences that share a common prefix that we refer to as \emph{regions}. 
Region $\region(\prefix, \legalactions; \staterewardtree)$ is the set of trajectories that start with prefix 
$\prefix = (a_0, \ldots, a_{t-1})$ and then take a next action from $\legalactions \subseteq \actions$.
The root region $\region(\varnothing, \actions)$ is the set of all trajectories.
An example search tree is shown in \figref{fig:state_reward_and_search_trees} (b).
The root (top) is the set of all trajectories
and its right child is the set of trajectories $\left\{\trajectory : a_0 \in \{2, 3\} \right\}$.

A search queue is initialized with the root region, and then the search tree is repeatedly expanded by choosing a region $\region = \region(\prefix, \legalactions; \staterewardtree)$ from the queue 
and a next action $a_t \in \legalactions$. $\region$ is split into two child regions.
The first appends $a_t$ to the prefix and allows any next action to follow; i.e., $\region_1 = \region(\prefix \concat a_t, \actions)$ where $\concat$ denotes concatenation.
The second leaves the prefix unchanged and eliminates $a_t$ as a possible next action; i.e., $\region_2 = \region(\prefix, \legalactions \backslash \{a_t\})$.
If $s_{\prefix \concat a_t}$ is a terminal state then $\region_1$ contains a single
trajectory and is not expanded further. 
If $\legalactions \backslash \{a_t\}$ is empty, then $\region_2$ 
can be discarded.
An interaction with the environment is generated only for the first new region, and the resulting state is stored so that it can be re-used by all other nodes sharing the same prefix.
In \figref{fig:state_reward_and_search_trees} (b), the first split chose $a_0=1$
and created regions $\region_1 = \region((1), \actions)$ and $\region_2 = \region(\varnothing, \actions \backslash \{1\})$.

\paragraph{Optimal completions. } 
For any region $\region(\prefix, \legalactions; \staterewardtree)$ popped from the queue, it is possible to optimally complete it with respect to $G_\theta$ without any backtracking in the search.
That is, letting $\prefix = a_0, \ldots, a_t$, we can compute $\argmax_{a_{t+1}, \ldots, a_T \mid a_{t+1} \in \legalactions} G_\theta(\prefix \concat (a_{t+1}, \ldots, a_T); \Gamma, \staterewardtree)$ using only $T - t$ interactions with the environment.
The key idea
is to define random variables $G_\theta(\trajectory; \gumbels, \staterewardtree)$ not only for full trajectories $\trajectory$ but also for every region in the search tree.
The random variable for a region is assigned to be the max over $G_\theta$ of all trajectories in the region: 
$\Gtheta(\region; \gumbels, \staterewardtree) 
= \max_{\trajectory \in \region} \Gtheta(\trajectory; \gumbels, \staterewardtree)$.
Since the marginal distributions of the region random variables can be computed efficiently,\footnote{By properties of Gumbel distributions, the marginals are $\Gtheta(\region; \gumbels, \staterewardtree) \sim \Gumbel(\log \modelP{\region \mid \staterewardtree})$ where
$\modelP{\region \mid \staterewardtree} = \sum_{\trajectory \in \region} \modelP{\trajectory \mid \staterewardtree}$.
It can efficiently be computed by pushing the sum inwards through the shared prefix:
    $\modelP{\region(\prefix, \legalactions; \staterewardtree) \mid \staterewardtree} = \prod_{t'=0}^{t-1} \modelp{a_{t'} \mid s_{(a_0, \ldots, a_{t'-1})}} \sum_{a \in \legalactions} \modelp{a \mid s_{\prefix}}.$
    }
 the top-down algorithm \cite{maddison2014astar} can be applied to sample child region random variables conditional on the parents.    
By always following the search tree downwards towards the child with maximum $\Gtheta(\region; \gumbels, \staterewardtree)$, we descend straight to the optimal completion.
Notably, if we follow this strategy starting at the root region, we sample $\topt$ using only $T$ environment interactions.
%
%
\begin{wrapfigure}{r}{0.5\textwidth}
\begin{minipage}{.48\textwidth}
\begin{algorithm}[H]
{\footnotesize
    \centering
    \caption{Top-Down Sampling $\trajectory$}\label{alg:topdown}
    \begin{algorithmic}[1]
        \State{{\bf In:} environment $env$, actions $\actions$, $\epsilon$.}
        \State{{\bf Out:} Stream of $(\trajectory, D_\theta(\trajectory))$ pairs.}
        \State $Q, \staterewardtree \gets \hbox{Queue}, \StateRewardTree$
        \State $Q.\push(\varnothing, \actions, \Gumbel(0))$
        \While {$Q$ is not empty}
            \State $\prefix, \legalactions, G \gets Q.\pop()$ 
            \State $a \gets \hbox{Sample } \modelp{a \mid s_{\prefix}} 1\{a \in \mathcal{B}\}$ 
            \State $s_{\prefix \concat a}, r_{\prefix \concat a} \gets env.\step(a, s_{\prefix})$
            \If {$\legalactions \backslash \{a\}$ is not empty}
                \State $\mu \gets \log \modelP{\region(\prefix, \legalactions \backslash \{a\}) \mid \staterewardtree}$
                \State $G' \gets \TruncGumbel(\mu, G)$ 
                \State $Q.\push(\prefix, \legalactions \backslash \{a\}, G')$ 
            \EndIf
            \If {$s_{\prefix \concat a}$ is terminal}
                \State {\bf yield} $(\prefix \oplus a, G + \epsilon R(\prefix \oplus a, \staterewardtree))$ 
            \Else
                \State $Q.\push(\prefix \oplus a, \actions, G)$ 
            \EndIf
        \EndWhile
    \end{algorithmic}
}
\end{algorithm}
\end{minipage}
\vspace{-15pt}
\end{wrapfigure}

\paragraph{Top-down sampling of trajectories. }
Putting the above two sections together and simplifying expressions results in Algorithm~\ref{alg:topdown}, a new variant of top-down sampling.
Note that the algorithm produces an endless stream of $(\trajectory, D_\theta(\trajectory))$ pairs (line 15) and does not specify the order in which nodes are popped from the queue (various choices are discussed below).
The algorithm begins by sampling $\Gtheta(\region)$ for the root region $\region$ that contains all trajectories (line 4).
Line 6 pops a node from the queue and line 7 samples the action $a_t$ associated with the child region with maximum $\Gtheta$.
Line 8 queries the environment for the $s_{t+1}$ and $r_t$ that result from
taking $a_t$ as the next action, and the result is stored until $\staterewardtree$ is reset.
Then regions are divided as described above (line 17 corresponds to $\region_1$; lines 9-13 correspond to $\region_2$), and upon creation of new regions, their $\Gtheta$ values are sampled conditional upon the parent's $\Gtheta$ value (lines 11, 17).

If $Q$ is a priority queue with priority $\Gtheta(\region)$, then the algorithm will yield pairs in descending order of $\Gtheta(\trajectory)$, which also means that $\topt$ will be found after $T$ node expansions. We assume regions are prioritized this way until the first yield so that line 4 in Algorithm~\ref{alg:dpg} produces $(\topt, \dobj(\topt; \epsilon))$.
We are then free to change the priority function as in the next subsection and reorder the queue.
However if we do not, then this can  generate  ``Gumbel Top-K'' \cite{kool2019stochastic} by running Algorithm~\ref{alg:topdown} with priority $\Gtheta(\region)$ and return the first $K$ results.
Algorithm~\ref{alg:topdown} is better for RL than other $A^\star$ sampling algorithms \cite{maddison2014astar, kim2016exact}, because the others would roll-out an entire trajectory for each region expanded and thus make inefficient use of interactions with the environment. We expand on these details in \appref{app:further_astar_details}.
%



\paragraph{Searching for large $\dobj$ using A$^\star$ sampling. }
The final algorithm prioritizes regions on the queue using the return achieved so far and (if available) an upper bound on the return-to-go. 
It is the same as Algorithm~\ref{alg:topdown}, except before pushing a region on the queue (lines 4, 12, 17), we compute a priority for a region based on all the terms in \eqref{eq:direct_obj}.
Let $L(\region) = \sum_{t'=0}^{t-1} r_{(a_0, \ldots, a_{t'-1})}$ be the reward accumulated so far by the prefix and $U(\region) \ge \sum_{t'=t}^T r_{(a_0, \ldots, a_{t'-1})}$ be an upper bound on the return-to-go for any trajectory in region $\region$.
Recall the $\Gtheta(\region)$ computed during the search is the maximum $\Gtheta$ for any trajectory in the region.
We can then upper bound $\dobj(\region; \epsilon)=\max_{\trajectory \in \region} \dobj(\trajectory; \epsilon) \le \Gtheta(\region) + \epsilon \cdot (L(\region) + U(\region))$. We can also prune regions from the search if their upper bound is worse than $\dobj(\trajectory; \epsilon)$ for the best $\trajectory$ found so far.
Using the upper bound as a priority yields a stochastic version of A$^\star$ search (i.e., it is A$^\star$ Sampling).
In practice, there is a large literature on heuristic search methods that relax optimality guarantees of $A^\star$ search in order to arrive at good solutions faster (see, e.g., \cite{Pearl1981HeuristicST,hansen2007anytime}).
We have found benefit to adapting these methods to the search for $\tdirect$. 
In particular, we adapt static weighted $A^\star$ search \cite{pohl1970heuristic} to our setting by modifying the priority to be $\Gtheta(\region) + \epsilon \cdot (L(\region) + \alpha U(\region))$ for $0 \le \alpha < 1$, though we expect other methods to also be fruitful.


\section{Experiments}
\label{sec:experiments}


\paragraph{Combinatorial Bandits. }

We experiment with combinatorial bandits and compare \ouralgshort~to Upper Confidence Bound (UCB) algorithms \cite{auer2002using,cesa2012combinatorial}.
The environment is defined by a graph $G=(V, E)$ where $V = \{1, \ldots, n\}$ is the set of nodes and $E \subseteq V \times V$ is the set of undirected edges. 
For each edge $e \in E$  a parameter $\mu_e$ determines per-edge rewards as $r_e \sim \Uniform(0, 2 \mu_e)$.
An agent queries the environment with tree $\mathcal{T}$ and receives  reward $r_\mathcal{T} = \sum_{e \in \mathcal{T}} r_e$. Fresh realizations of $r_e$ are drawn for each episode.
UCB algorithms end an episode after a single interaction, while \ouralgshort~uses multiple interactions per episode (at the cost of seeing fewer realizations).
We compare to a "semi-bandit" version of UCB that observes more information (per-edge contributions to rewards) and a "full bandit" version that receives the same observations as \ouralgshort, the total reward $r_{\mathcal{T}}$ after producing a full tree. Note the similarity of the full bandit version to, e.g., CUCB \cite{chen2013combinatorial}.

To apply \ouralgshort, we let $\trajectory$ be a sequence of $|E|$ binary decisions of whether to include each edge in the spanning tree. Learnable parameters $\theta_e$ determine the probability of inclusion via $\sigma(\theta_e)$ where $\sigma$ is the sigmoid function.
The environment presents a legal set of actions at each step (see \appref{app:combinatorial_bandits} for details).
To compute $\tdirect$, we give a budget of 100 interactions and use priority $\Gtheta(\trajectory)$ in the search, enabling the early termination option in Algorithm~\ref{alg:dpg}.
\begin{figure*}
    \centering
    \begin{tabular}{ccc}
    \includegraphics[width=.23\textwidth]{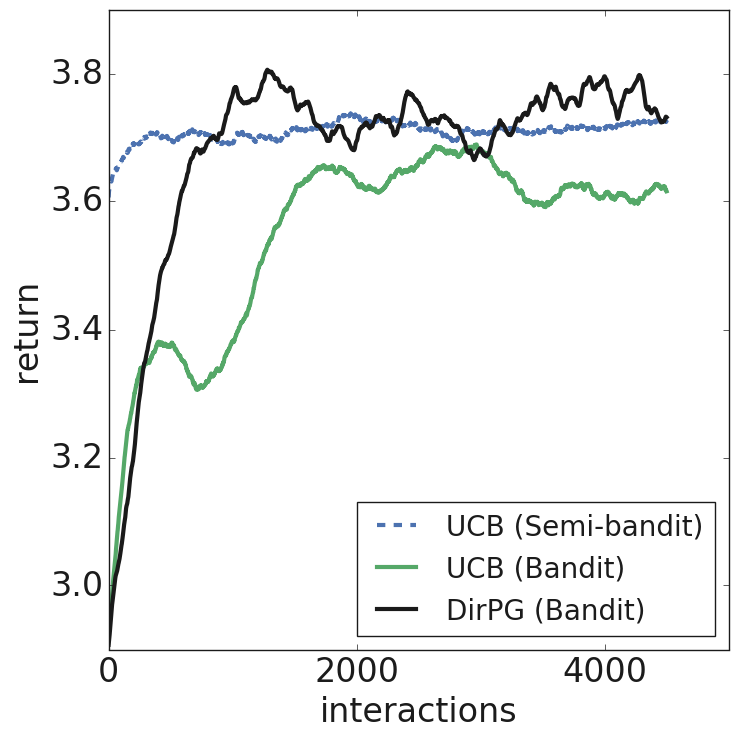} & 
    \includegraphics[width=.24\textwidth]{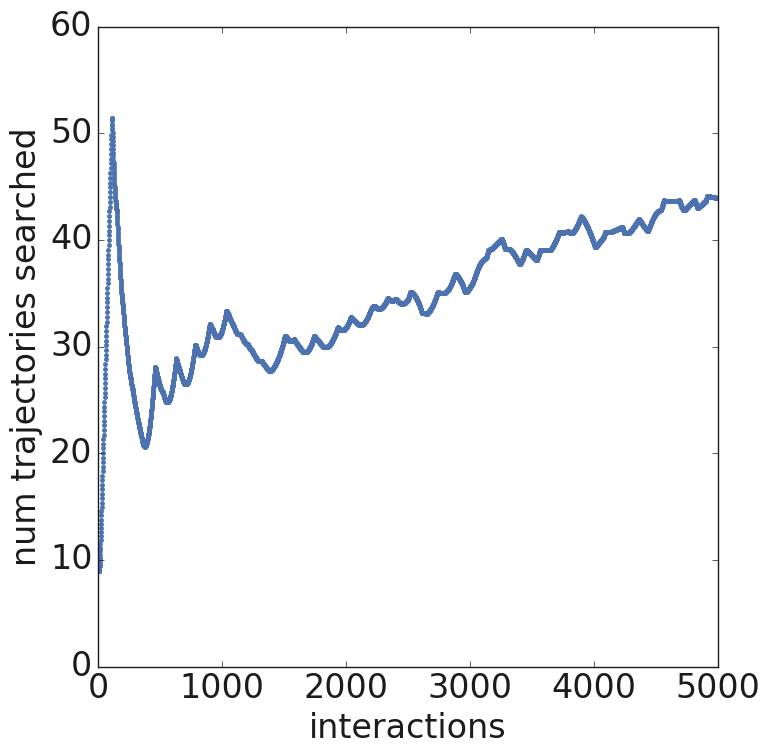} &
    \includegraphics[width=.25\textwidth]{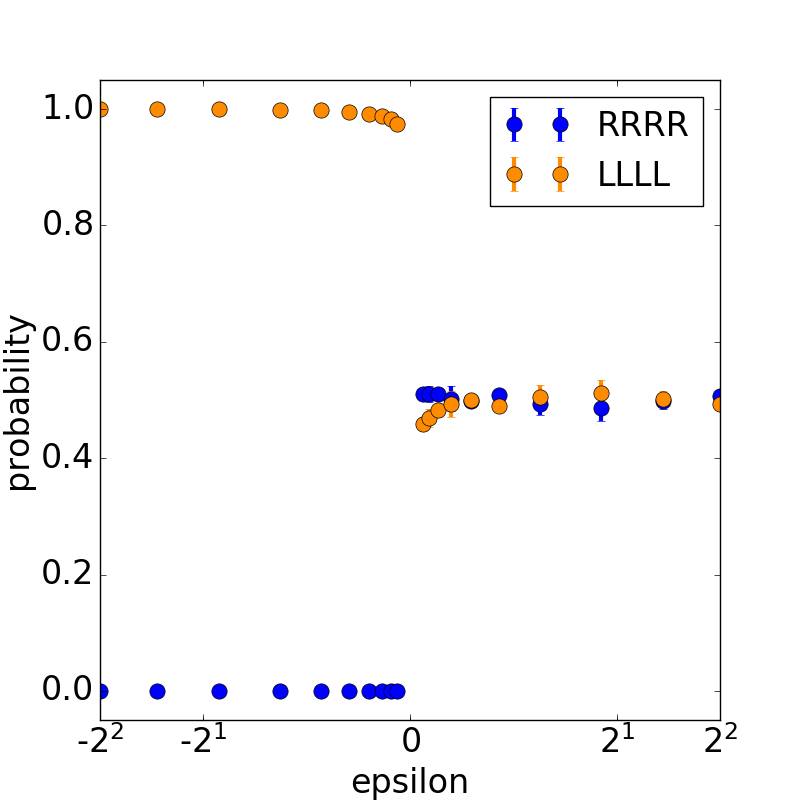} \\
    (a) & 
    (b) & 
    (c)
    \end{tabular}
    \vspace{-5pt}
    \caption{
    {
    Bandits and risk sensitivity. 
    (a) Average return vs \# of interactions.
    (b) Number of steps needed to find $\tdirect$.
    (c) DeepSea results showing learned $\Pi$(LLLL) (safe) and $\Pi$(RRRR) (risky) vs $\epsilon$.
    }}
    \label{fig:combbandits}
    \vspace{-10pt}
\end{figure*}
Results appear in Fig.~\ref{fig:combbandits} (a), which
shows the moving average return versus number of interactions, averaged over 10 runs. The \ouralgshort~curve is for samples of $\topt$, which is noisier due to there being fewer realizations.
\ouralgshort~is competitive with a UCB variant using more information, and it outperforms the comparable variant. Fig.~\ref{fig:combbandits} (b) shows the number of steps taken to find an improvement. Aside from initial noise due to the moving average, the number of interactions used in the search automatically grows as learning progresses.

\paragraph{DeepSea. }
Previously $\epsilon$ was considered a nuisance parameter, but we show that it controls an agent's preference for risk-seeking (positive $\epsilon$) versus risk-avoiding (negative $\epsilon$) behavior.
Analysis making this claim precise and a further experiment appears in \appref{app:risk}.

We use an adaptation of the DeepSea environment that was used by \cite{o2018variational} to study risk sensitivity.
The environment is a 5x5 grid where the agent starts from the top-left cell and the goal is in the bottom-right.
The agent has a choice of left (L) or right (R) at each step.
If the agent chooses L, it gets 0 reward and moves down and left. If it chooses R, it gets a reward sampled from ${\cal N}(1, 1)$ if transitioning to the bottom-right corner and otherwise $-\frac{1}{3}$.
This is interesting because any policy that is a mixture of LLLL and RRRR has optimal return (mixture of 0, ${\cal N}(0, 1)$ respectively), but the policies have different variance and thus we expect the choice of $\epsilon$ to affect what the agent learns.

In \figref{fig:combbandits} (c) we train policies with a range of $\epsilon$ values for $400,000$ episodes to ensure convergence and plot the probability assigned to trajectories LLLL and RRRR in the learned policy. 
For $\epsilon < 0$, most mass is put on LLLL, which has no variance and is thus favorable to a risk-avoiding agent.
For $\epsilon > 0$, mass is split evenly, which has highest ``controllable risk'' (see \appref{app:risk}). 

\paragraph{MiniGrid. }

In our final experiments we use the {\bf MiniGrid-MultiRoom-N6-v0} environment \cite{babyai_iclr19} to study how to prioritize nodes within the search for $\tdirect$.
MiniGrid is a partially observable grid-world where the agent observes an egocentric $7 \times 7$ grid around its current location and has the choice of 7 actions including moving right, left, forward, or toggling doors. We use environments of $25 \times 25$ grids with a series of 6 connected rooms separated by doors that need to be opened. Intermediate rewards are given for opening doors and reaching a final goal state. As baselines we compare to REINFORCE and the cross entropy method. In all of the methods we utilized the simulator to reset the environment so that multiple trajectories could be sampled starting from the same environment seed.
In all cases, we use a total of 3000 interactions per environment seed (episode). In our method, we use 100 interactions to sample $\topt$ (the trajectory length) and 2900 interactions to search for $\tdirect$. In REINFORCE and in the cross entropy method we sample 30 independent trajectories, where each is 100 interactions long.
Details on their implementation are in \appref{app:minigrid}.
\begin{figure*}
    \centering
    \begin{tabular}{ccc}
         \hspace{-5pt}
         \includegraphics[width=.3\textwidth]{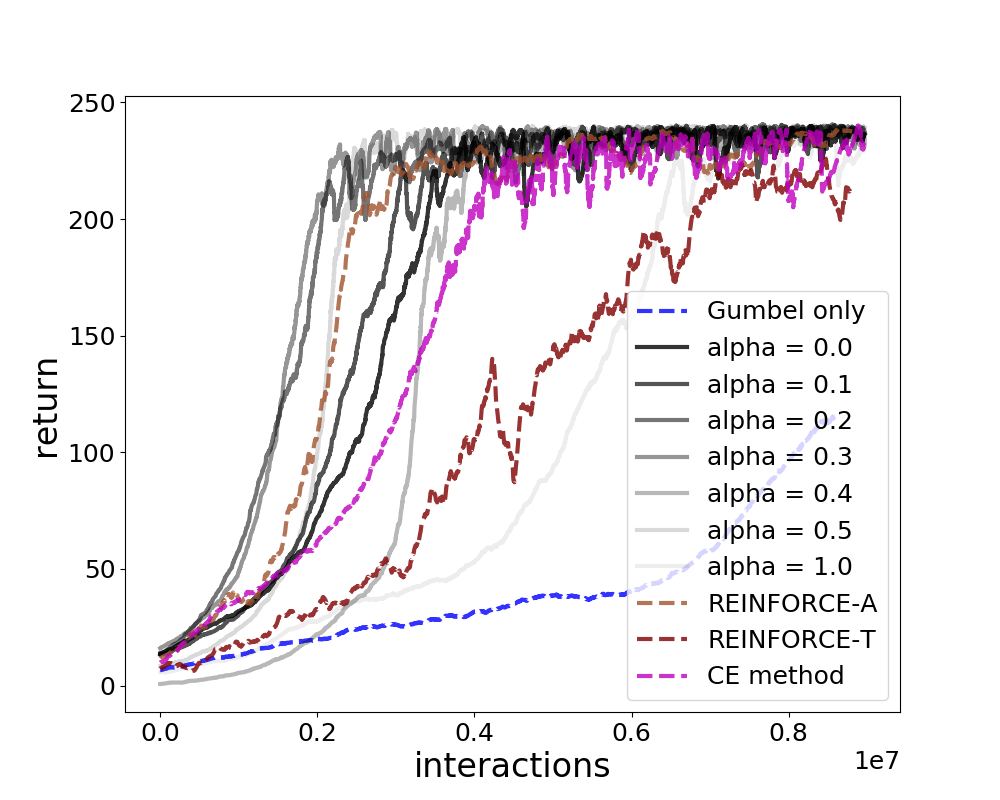} &  
         \hspace{-15pt}
         \includegraphics[width=.3\textwidth]{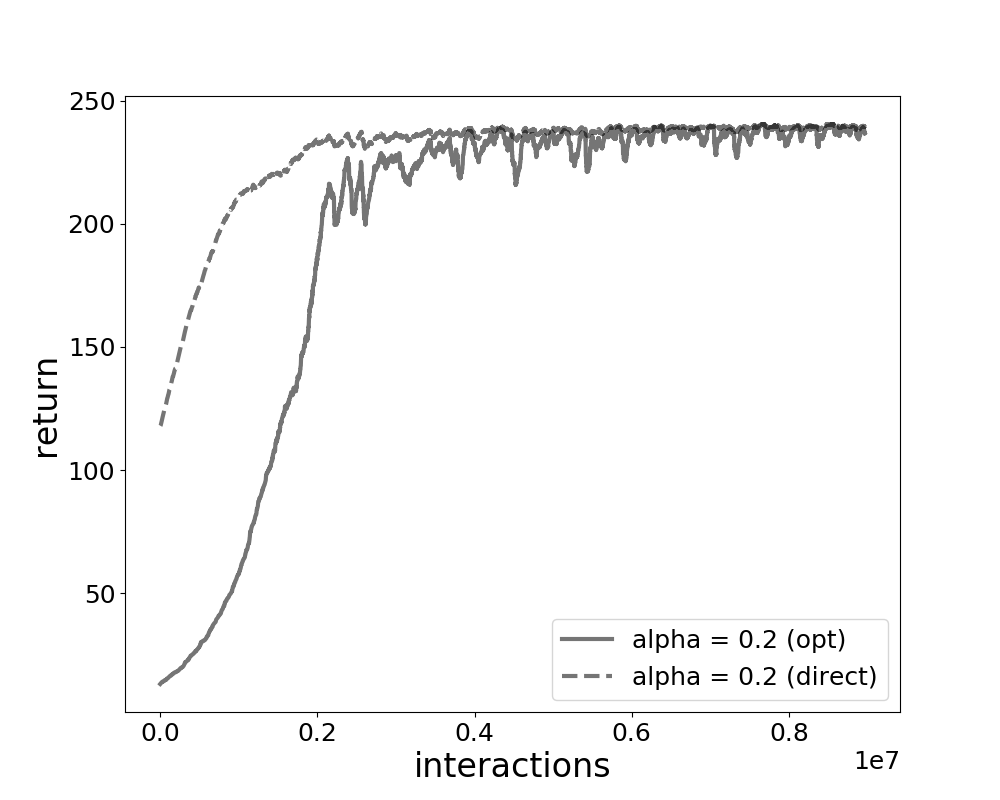} &
         \hspace{-15pt}
         \includegraphics[width=.3\textwidth]{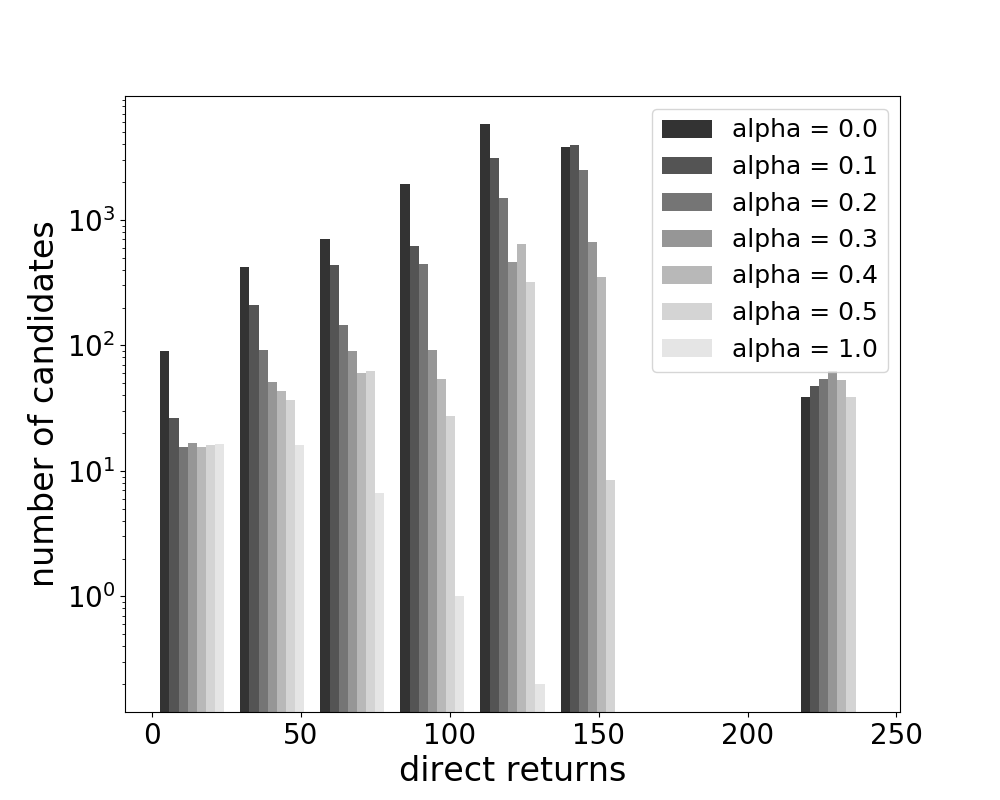}
         \\
         (a) & 
         (b) & 
         (c)
    \end{tabular}
    \vspace{-5pt}
    \caption{{\footnotesize Minigrid results.
    (a) Return vs number of interactions.
    (b) Direct objective of $\tdirect$ and $\topt$ vs iteration.
    (c) Histograms showing quality-vs-quantity tradeoff for various search priorities.
}}
\vspace{-10pt}
    \label{fig:minigrid1}
\end{figure*}

We explore variations on how to set the priority of nodes in the search for $\tdirect$.
First, in the ``Gumbel only'' priority, we use just $\Gtheta(\region)$ as a region's priority.
In the others, we use $\Gtheta(\region; \staterewardtree, g) + \epsilon (L(\region) + \alpha U(\region))$,
where $U$ is based on the Manhattan distance to the goal and the number of unopened doors.
Setting $\alpha=0$ trades off enumerating by descending order of $\Gtheta(\region; \staterewardtree, g)$ with favoring prefixes that have already achieved high return.
Setting $\alpha=1$ yields A$^\star$ search.
\figref{fig:minigrid1} (a) shows average return versus training episode.
$\alpha=0$ provides good results, and increasing $\alpha$ up to $\alpha=0.3$ gives improved performance.
Beyond that, performance degrades, with $\alpha=1$ performing worst.

To better understand this, we partially trained a model for 1.2M interactions
and then froze the parameters and ran several searches
for the same number of interactions but with different priority functions.
\figref{fig:minigrid1} (c) shows the results.
For smaller $\alpha$, more trajectories are finished to completion but the returns achieved are worse.
As $\alpha$ increases, fewer full trajectories are found but they have better returns, but past $\alpha = 0.4$ not enough full trajectories are found, and both the quality and the quantity shrink.
Thus, setting $\alpha$ too high leads to ``breadth-first behavior'' where too much time is spent exploring prefixes and not completing trajectories.
In \figref{fig:minigrid1} (b), we show the relationship between $\dobj(\tdirect)$ and $\dobj(\topt)$ over the course of learning.
This shows that $\tdirect$ does not need to find a trajectory with the optimal return in order to provide signal for the policy to improve.

\section{Related Work}
\label{sec:related_work}

Similarities can be drawn to the body of work casting RL as probabilistic inference, in particular in Expectation-Maximization (EM) Policy Search methods
\cite{Peters10,toussaint2006probabilistic,rawlik12,levine2013variational,levine2016end,montgomery2016guided,chebotar2016path,abdolmaleki2018maximum, buesing2018woulda}. Broadly, these methods alternate a step akin to posterior inference that improves a trajectory distribution with an update to the policy parameters using in an EM formulation.
In this context our work would be most similar to an incremental variant \cite{neal1998view} of Monte Carlo EM \cite{levine2001implementations}, 
though \ouralgshort~has significant differences, including the use of A$^\star$ sampling to guide the sampling and the use of direct optimization, which can be interpreted as a variance reduction strategy. 
We discuss this in detail in \appref{app:further_analysis}.

The initial \ouralgshort~reparameterization is similar to \cite{heess2015learning}, but the setting and approach are very different.
The most prominent example of search in RL is Monte Carlo Tree Search (MCTS) \cite{kocsis2006bandit,browne2012survey}.
MCTS is quite different because---unlike \ouralgshort---it uses search and a simulator at test time, but it becomes closer when search results are distilled into a policy as in \cite{silver2018general}. However, we are not aware of results showing that MCTS can be used to directly compute a policy gradient. 
One can imagine an MCTS-style algorithm that explores $k$ trajectories under a fixed realization of environment noise and chooses the one with highest return, then distills into a policy via a gradient update to increase the probability assigned to the chosen trajectory. As $k \rightarrow \infty$, this will approach the optimal \ouralgshort~update with $\epsilon \rightarrow \infty$. Bbased on risk sensitivity results in \appref{app:risk}, we can see that this algorithm will be very risk-seeking. Thus, \ouralgshort~offers a degree of control via $\epsilon$ that isn't available to this MCTS-style counterpart.

Another related use of search trees is the \emph{vine} method from \cite{schulman2015trust}, which leverages a simulator's ability to reset to previous states to construct a tree over trajectories. Multiple roll-outs are created from tree nodes, and common random numbers are used across the roll-outs to reduce variance.

\section{Discussion}

We have presented a new method for computing a policy gradient and studied its properties from theoretical and empirical perspectives.
This also provides new understandings of direct loss optimization in terms of variance reduction and risk-sensitivity. 
One limitation is that in its current form, the algorithm only learns in an episodic framework and from complete trajectories. 
We are currently exploring how this limitation could be removed.
Our experiments so far have been geared towards understanding the algorithm and its important degrees of freedom.
We are eager to take these learnings and apply them to real-world applications where search and heuristics (upper bounds) have traditionally been successful like navigation, combinatorial optimization, and program synthesis.
\section*{Broader Impact}
This work presents a general theoretical and algorithmic contribution to reinforcement learning (RL) research.
One contribution (\appref{app:risk}) is an analysis of the risk-sensitive behavior of the algorithm as parameter $\epsilon$ is varied. 
This provides an axis of control beyond simply maximizing expected future reward, which is likely a beneficial analysis to perform (though far-removed from well-defined impacts).
We'll refrain from commenting on the future societal consequences of general advances in RL research, because this work is more theoretical and conceptual in nature, and it is a complex topic that is better covered in the context of work that is closer to specific impacts.

\begin{ack}
\end{ack}

\bibliographystyle{plain}
\bibliography{references}

\clearpage
\newpage
\appendix
\section{Further Analysis}
\label{app:further_analysis}

In this section we broaden understanding of the \ouralgshort{} update by developing an alternate interpretation of \ouralgshort{} as the gradient of some other function, which we discovered by reverse-engineering
the update. This provides insight into the precise effect of $\epsilon$, provides an interpretation of \ouralgshort{} as having a built-in control variate, and allows relating the algorithm to other areas of reinforcement learning.

\subsection{Reverse Engineering an Objective Function}
The final objective that we arrived at via reverse engineering is
\begin{align}
l(\theta, \epsilon) &= \expect_{\staterewardtree \sim P}\left[\frac{1}{\epsilon}  \log\left(\expect_{\trajectory \sim \modelP{\cdot \mid \staterewardtree}}\left[ \exp(\epsilon R(\trajectory, \staterewardtree)) \right]\right)\right]. \label{eq:marginal_logprob}
\end{align}
Here we show that differentiating it indeed leads to the \ouralgshort{} update.
To derive the \ouralgshort{} update, first divide by 1:
\begin{align}
l(\theta, \epsilon) = \frac{1}{\epsilon} \expect_{\staterewardtree \sim P} \left[\log \frac{\sum_{\trajectory} \exp \left\{ \log \modelP{\trajectory \mid \staterewardtree} + \epsilon R(\trajectory, \staterewardtree) \right\} }
                               {\sum_{\trajectory} \exp \left\{ \log \modelP{\trajectory \mid \staterewardtree} \right\}} \right], \label{eq:marginal_logprob_div1}
\end{align}
and then differentiate to get
\begin{align}
\nabla_{\theta} l(\theta, \epsilon) &=  \frac{1}{\epsilon} \mathbb{E}_{\staterewardtree \sim P} \left[ \expect_{\trajectory \sim P_R(\cdot \mid S)} \left[ \nabla_\theta \log \modelP{\trajectory \mid \staterewardtree} \right]  \right]
-  \frac{1}{\epsilon}  \mathbb{E}_{\staterewardtree \sim P} \mathbb{E}_{ \trajectory \sim \modelP{\cdot \mid \staterewardtree}} \left[ \nabla_\theta \log \modelP{\trajectory \mid \staterewardtree} \right]. \label{eq:dlpg_alt}
\end{align}
where $P_R(\actiontrajectory \mid \staterewardtree) \propto \modelP{\actiontrajectory \mid \staterewardtree} \exp(\epsilon R(\actiontrajectory, \staterewardtree))$.
Now we can reparameterize the expectations in \eqref{eq:dlpg_alt} using Gumbel-max and express the samples in terms of \eqref{eq:tstar}:
\begin{align}
&= \frac{1}{\epsilon} \expect_{\staterewardtree \sim P} \left[ \mathbb{E}_{\gumbels} \left[ \nabla_\theta \log \modelP{\trajectory^*(\epsilon) \mid \staterewardtree} \right]  \right] - \frac{1}{\epsilon} \expect_{\staterewardtree \sim P} \left[  \mathbb{E}_{\gumbels}  \left[ \nabla_\theta \log \modelP{\trajectory^*(0) \mid \staterewardtree} \right] \right]. \label{eq:reparamerized_control_variate}
\end{align}
Having expressed both expectations in terms of Gumbel noise $\gumbels$ with the same distribution, we can use common random numbers to recover the direct policy gradient:
\begin{align}
& =
\frac{1}{\epsilon} \mathbb{E}_{\staterewardtree \sim P, \gumbels} \left[
    \nabla_\theta \log \modelP{\tstar{\epsilon} \mid \staterewardtree}
    - \nabla_\theta \log \modelP{\tstar{0} \mid \staterewardtree}\right]. \label{eq:dlpg_app}
\end{align}
 The final result is the \ouralgshort{} gradient, and note that there are no approximate equalities here: \eqref{eq:marginal_logprob} is in some sense the underlying objective that \ouralgshort{} optimizes when $\epsilon$ is treated as a hyperparameter.

\subsection{Control Variate Interpretation. }
\label{app:variance_reduction}
The $\mathbb{E}_{ \trajectory \sim \modelP{\cdot \mid \staterewardtree}} \left[ \nabla_\theta \log \modelP{\trajectory \mid \staterewardtree} \right]$ term of \eqref{eq:dlpg_alt} is the expected value of a score function and thus is identically equal to $\boldsymbol{0}$. 
There would be no benefit of including the term in \eqref{eq:reparamerized_control_variate}.
The benefit of including it only becomes apparent in \eqref{eq:dlpg_app}, where we can interpret it as a control variate.
The optimization problems that define $\tdirect$ and $\topt$ differ only in value of $\epsilon$, so for small $\epsilon$ we expect the solutions to have similar features and correlated score functions.
When this is the case, control variates reduce the variance of the overall gradient estimate. 
To our knowledge, direct optimization has not previously be understood in these terms.

\paragraph{Further experiment on effect of control variate on variance. } 
We measured the variance of DirPG updates with and without the control variate term of Eq.~\ref{eq:reparamerized_control_variate} during training on MiniGrid. See \figref{fig:ctrl_var}.  The control variate reduces variance, particularly later in training, when $\boldsymbol{a}_{opt}$ is better correlated with the reward function.
\begin{figure}
    \centering
    \includegraphics[width=.4 \linewidth]{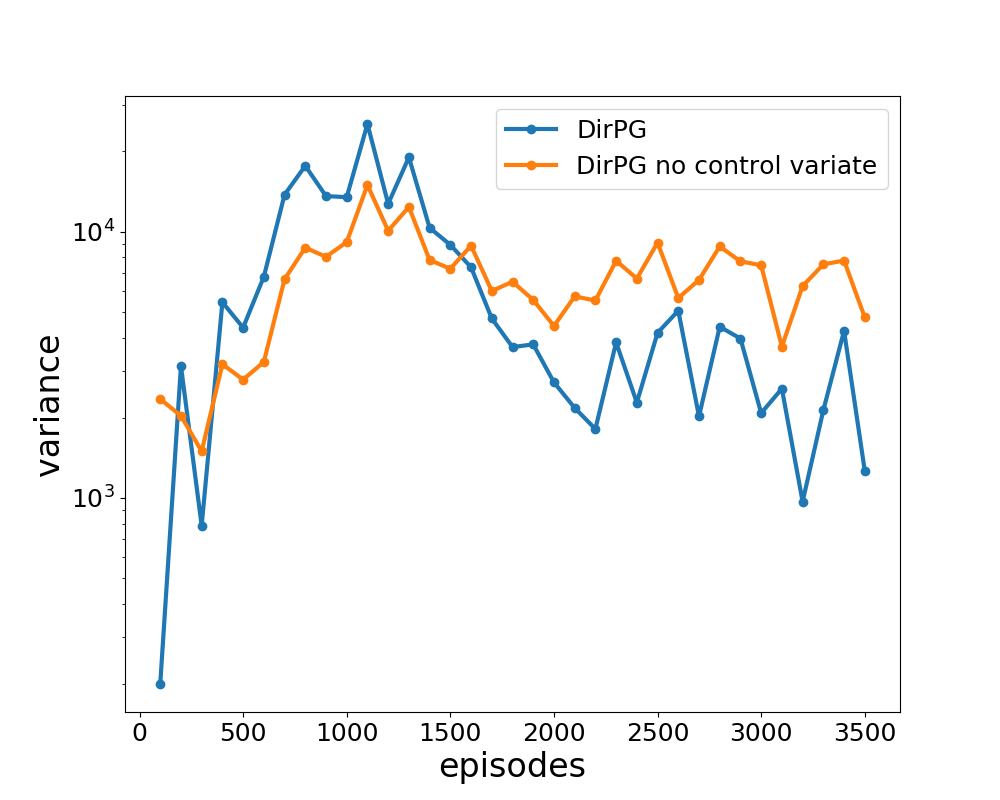}
    \caption{Total empirical variance of the DirPG update as a function of the number of training episodes on MiniGrid.}
    \label{fig:ctrl_var}
\end{figure}

\subsection{Risk Sensitivity}
\label{app:risk}

\subsubsection{Relation to Risk-Sensitive Control}
The objective \eqref{eq:marginal_logprob} is closely related to a classical objective in risk-sensitive control \cite{pratt1964risk,howard1972risk,coraluppi1997optimal},
$\log\expect\left[ \exp(\epsilon R(\trajectory, \staterewardtree)) \right]/ \epsilon$. For $\epsilon > 0$, optimal policies under the classical objective prefer high risk strategies as long as high rewards have some positive probability. For $\epsilon < 0$, optimal policies prefer low risk strategies that avoid placing probability on low rewards.  \eqref{eq:marginal_logprob} has an important difference. Following \cite{coraluppi1997optimal,maddison2017particle}, we take a Taylor  expansion of $\exp(t)$ and $\log(1+t)$ at $t=0$ to get
\begin{align}
l(\theta, \epsilon) &= 
\expect_{\staterewardtree \sim P, \trajectory \sim \modelP{\cdot \mid \staterewardtree}}[R(\trajectory, \staterewardtree)] 
+ \frac{\epsilon}{2} \expect_{\staterewardtree \sim P} [\mathrm{var}_{\trajectory \sim \modelP{\cdot \mid \staterewardtree}}(R(\trajectory, \staterewardtree))]   + \mathcal{O}(\epsilon^2), \label{eq:risk_sensitivity_taylor}
\end{align}
where we use the notation $\mathrm{var}_{\trajectory \sim \modelP{\cdot \mid \staterewardtree}}(R(\trajectory, \staterewardtree))$ to mean the conditional variance of $R(\trajectory, \staterewardtree)$ given $\staterewardtree$. 
Note that expected conditional variance is not equal to the joint variance, which makes this objective different from the typical risk-sensitive analysis. If the second term were simply the variance under the joint, then the agent is sensitive to variance in return regardless of whether it was due to stochasticity in the environment or in the policy. In \eqref{eq:risk_sensitivity_taylor}, we see that the agent only seeks out or suppresses ``controllable risk,'' which is variance in return created due to stochasticity in its policy.

\paragraph{Further experiment on ``controllable risk.'' }
To further illustrate this, we used numerical integration to compute \eqref{eq:marginal_logprob} for a simplified ``Gaussian choice'' setting where an agent chooses to take a reward sampled from $\mathcal{N}(0, 1)$ with probability $p$ and 0 reward with probability $1 - p$. \figref{fig:quadrature_app} shows that the risk-seeking objective favors ``controllable risk'' created due to stochasticity in the agent's policy but not variance created due to stochasticity in the environment.

\begin{figure*}
    \centering
    \includegraphics[width=.35\textwidth]{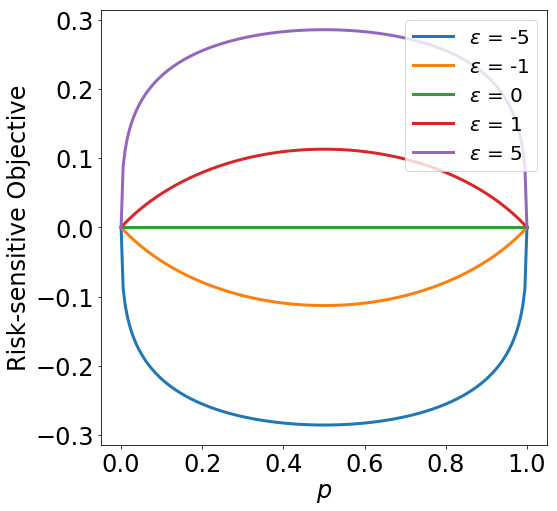} \\
    \caption{
    Quadrature evaluation of \eqref{eq:marginal_logprob} for the Gaussian choice problem for varying $\epsilon$.
    }
    \label{fig:quadrature_app}
\end{figure*}
\section{Approximate Optimization of $\tdirect$}
\label{app:approximate_optimization}

\paragraph{Proof of Correctness of Gumbel-approx-max in Deterministic Multi-armed Bandits}

Suppose we have $N$ arms, each with a fixed but unknown reward $R(i)$ and that arms are ordered 
according to their reward so $R(i) > R(j)$ iff $i > j$, and $\epsilon > 0$. Let the following:
\begin{itemize}
    \item $\modelp{i} \propto \exp \theta_i$ be the probability of arm $i$ under a softmax policy parameterized by $\theta$,
    \item $\Gtheta(i) \sim \Gumbel(\theta_i)$
    \item $\dobj(i, \epsilon) = \Gtheta(i) + \epsilon R(i)$ be the direct objective
    \item $i_{opt} = \argmax_i \dobj(i, 0)$
    \item $i_{dir} = \argmax_i \dobj(i, \epsilon)$
\end{itemize}
Finally, let $i_{approx}$ be the value of $i_{direct}$ arising from running Algorithm~\ref{alg:dpg} using $\Gtheta(i)$ as priority.
That is, we iterate over $i$ in descending order of $\Gtheta(i)$ until we find an $i$ such that $\dobj(i, \epsilon) > \dobj(i_{opt}, \epsilon)$
or we have enumerated all $i$, in which case we set $i_{approx} = i_{opt}$.

We prove that learning using $i_{approx}$ in place of $i_{dir}$ still leads to learning the optimal policy.

\begin{lemma}
$i_{direct} \geq i_{approx} \geq i_{opt}$.
\end{lemma}
\begin{proof}
To prove $i_{approx} \geq i_{opt}$, observe that by definition we have $\dobj(i_{approx}, \epsilon) \geq \dobj(i_{opt}, \epsilon)$
and $\Gtheta(i_{opt}) \geq \Gtheta(i_{approx})$ 
This implies
\begin{align}
\Gtheta(i_{approx}) + \epsilon R(i_{approx}) & \geq \Gtheta(i_{opt}) + \epsilon R(i_{opt}) \\
\epsilon R(i_{approx}) - \epsilon R(i_{opt}) & \geq \Gtheta(i_{opt}) - \Gtheta(i_{approx}) \geq 0.
\end{align}
Thus $R(i_{approx}) \geq R(i_{opt})$ and $i_{approx} \geq i_{opt}$.

To prove $i_{dir} \geq i_{approx}$ observe that we must have $\Gtheta(i_{approx}) \geq \Gtheta(i_{dir})$, because otherwise we would have encountered $i_{dir}$ before $i_{approx}$ when iterating $i$'s, and because $\dobj(i_{dir}, \epsilon) \geq \dobj(i_{approx}, \epsilon)$ by definition, we would have chosen $i_{dir}$ as $i_{approx}$ when we encountered it.

So we have $\Gtheta(i_{approx}) - \Gtheta(i_{dir}) \geq 0$, which implies
\begin{align}
\Gtheta(i_{dir}) + \epsilon R(i_{dir}) & \geq \Gtheta(i_{approx}) + \epsilon R(i_{approx}) \\
\epsilon R(i_{dir}) - \epsilon R(i_{approx}) & \geq \Gtheta(i_{approx}) - \Gtheta(i_{dir}) \geq 0  \\
\end{align}
Thus $R(i_{dir}) \geq R(i_{approx})$ and $i_{dir} \geq i_{approx}$.
\end{proof}

\begin{lemma}
We're at a stationary point iff $i_{direct} = i_{opt}$ (or $i_{approx} = i_{opt}$) almost surely.
\end{lemma}
\begin{proof}
In one direct, if $i_{direct} = i_{opt}$ almost surely, then \ouralgshort{} updates on 0 almost surely. In the other direction,
suppose for the sake of contradiction that there is some realization of $\Gtheta$ where $i_{direct}$ is not equal to $i_{opt}$.
By Lemma 1, $i_{direct} > i_{opt}$.
Then the gradient vector will have a positive entry for $\theta_{i_{direct}}$ and a negative entry for $\theta_{i_{opt}}$.
In order to be at a stationary point, other realizations of $\Gtheta$ need to cancel these contributions.
Because of Lemma 1, however, it is only possible to simultaneously decrement the gradient vector at $i$ and increment it at $j$ if $j > i$.
The only way to decrement the previously incremented entry for $i_{direct}$ would be to increment an even larger entry,
and the only way to increment the previously decremented entry for $i_{opt}$ would be to decrement an even smaller entry.
Thus, there is no way to cancel gradients if any entry is nonzero, and thus the only way to get a zero gradient is if $i_{direct} = i_{opt}$ for all realizations of $\Gtheta$. In Lemma 1 we have $i_{direct} \geq i_{approx} \geq i_{opt}$, so the same argument holds for $i_{approx}$.
\end{proof}

\begin{proposition}
The stationary points assuming exact optimization of $i_{direct}$ are the same as the stationary points assuming approximate optimization to get $i_{approx}$.
\end{proposition}
\begin{proof}
By Lemma 2, all stationary points assuming exact optimization have $i_{direct} = i_{opt}$ for all realizations of $\Gtheta$.
By Lemma 1, in each of these realizations we have $i_{direct} \geq i_{approx} \geq i_{opt}$. 
Thus, for all realizations we have $i_{approx} = i_{opt}$ and thus we are at a stationary point assuming approximate search. In the other direction, Lemma 2 implies that all stationary points assuming approximate optimization have $i_{approx} = i_{opt}$ almost surely. The only way for this to happen is that in trying to find $i_{approx}$ we exhaustively iterated over all arms and found no improvement. Thus, $i_{direct}$ could not have been an improvement and $i_{direct} = i_{opt}$ almost surely.
\end{proof}
\section{Further Details on $A^\star$ sampling trajectories}
\label{app:further_astar_details}

Here we provide a more detailed version of \secref{sec:topdown}, 
which allows us to more precisely state the limitations of the original $A^\star$ sampling algorithm for RL, and how our algorithm fixes the problems.

\paragraph{Gumbel Processes.}
To evaluate $\dobj(\trajectory, \epsilon)$, which defines $\topt$ and $\tdirect$,
we need to sample a $\Gtheta(\trajectory)$ value for each complete trajectory encountered
during the search.
It is not possible to generate $\Gtheta(\trajectory)$ for each $\trajectory$ before starting the
search, because there may be exponentially (or even infinitely) many possible trajectories.
Another option would be to expand the search tree independently of $\Gtheta$ values and then sample
$\Gtheta(\trajectory)$ via \eqref{eq:G} for each singleton region encountered during the search.
This would produce $\Gtheta$ values with the right distribution, but it is also a non-starter
because we are precisely interested in biasing the search towards trajectories with large $\Gtheta$ values.

The solution to this problem comes from Maddison et al.
Instead of only assigning $\Gtheta$ values to trajectories, we also assign them to regions.
To assign random variables to overlapping regions in a
consistent way, Maddison et al.~introduce the \emph{Gumbel Process}.
A Gumbel process is defined in terms of a sample space $\Omega$ and measure $\mu$.
In our case, $\Omega = \actions^T$ is the set of all length $T$ trajectories and $\mu$ assigns probabilities
to any subset $\region \subseteq \actions^T$ as $\mu(\region \mid \staterewardtree) = \sum_{\trajectory \in \region} \modelP{\trajectory \mid \staterewardtree}$. 
A Gumbel Process is then defined as the set $\{G(\region) \mid \region \subseteq \Omega\}$ where the following
properties hold:
\begin{enumerate}
    \item $G(\region) \sim \Gumbel(\log \mu(\region))$,
    \item $\region_1 \cap \region_2 = \emptyset \implies G(\region_1) \perp G(\region_2)$,
    \item $G(\region_1 \cup \region_2) = \max(G(\region_1), G(\region_2))$.
\end{enumerate}
That is, (1) the $G$ values are marginally distributed as Gumbels with location given by the log
measure of the region, (2) the random variables for disjoint
regions are independent, and (3) the random variable in the union of two regions is equal to the
max of the random variables in the two regions.
A fourth property is implied by the first three, which we state in our context:
\begin{enumerate}
\setcounter{enumi}{3} 
\item $X(\region) = \argmax_{\trajectory \in \region} G(\trajectory) \sim 1\{ \trajectory \in \region\} \modelP{\trajectory \mid \staterewardtree}$.
\end{enumerate}
That is, the argmax trajectory $X(\region)$ in a region  is distributed according to $\modelP{\cdot \mid \staterewardtree}$ that is
masked out to only give support to $\region$.
Finally, an important property that comes from Gumbel distributions is that
$G(\region)$ and $X(\region)$ are independent random variables \cite{maddison2014astar}.
This means that we are free to interleave the sampling of $X$ and $G$ as we
please, and it will be leveraged in the algorithms in the following sections.

\paragraph{Top-Down Sampling.}
Conceptually, if we had sampled $\Gtheta(\trajectory)$ for all $\trajectory$, then the rest of the
Gumbel process would be determined by $\Gtheta(\region) = \max_{\trajectory \in \region} \Gtheta(\trajectory)$.
However, Maddison et al.~show that assuming $\mu$ is computable for all regions, a
Gumbel Process can be constructed lazily in a ``top-down'' fashion, first sampling $G(\Omega)$, and then
recursively subdividing regions $\region_0$ and sampling $G$'s for the child regions conditional upon the value
of $G(\region_0)$.
Specifically, they divide $\region_0$
into three disjoint regions: $\region_1, \region_2$, and $\{X(\region_0)\}$
such that $\region_0 = \region_1 \cup \region_2 \cup \{ X(\region_0) \}$.
They show that for $i \in \{1, 2\}$ the conditional distribution of $G(\region_i)$ given previous splits in the tree is $\TruncGumbel(\log \mu(\region_i), G(\region_0))$ and $G(\{X(\region_0)\}) = G(\region_0)$.

Under our choice of regions, $\mu(\region \mid \staterewardtree) = \sum_{\trajectory \in \region} \modelP{\trajectory \mid \staterewardtree}$ 
can indeed be computed efficiently as
\begin{align}
    \modelP{\region(\prefix, \legalactions; \staterewardtree) \mid \staterewardtree} =
    & \left( \prod_{t'=0}^{t-1} \modelp{a_{t'} \mid s_{(a_0, \ldots, a_{t'-1})}} \right) \sum_{a \in \legalactions} \modelp{a \mid s_{\prefix}}.
\end{align}

$\legalactions$ is the set of actions that can be taken after the prefix $\prefix$. 

If all prefixes eventually terminate with probability 1, then it is possible 
to apply one step of Top-Down Sampling to sample trajectories.
To split a region $\region_0 = \region(\prefix, \legalactions)$, we would sample
$X(\region_0) \sim 1\{\trajectory \in \region_0\} \modelp{\trajectory \mid s_{\prefix}}$.
This is straightforward because it is essentially conditioning on a prefix in
an autoregressive model. Specifically, start with $\prefix$,
sample $a_t \sim 1\{a_t \in \legalactions\} \modelp{a_t \mid s_{\prefix}}$,
and then sample a completion according to 
\begin{align}
    \prod_{t' = t+1}^T \modelp{a_{t'} \mid s_{(a_0, \ldots, a_{t'-1})}}
\end{align}

However, recursing would be problematic because we do not have a way of splitting
$\region_0 \backslash \{X(\region_0)\}$ into two regions that can compactly be
represented as a prefix plus legal set of next actions.
To address a similar issue, Kim et al.~propose a modified split criteria that
divides a region $\region_0$ into two regions. Roughly the idea is to group together
$\region_1 \cup \{X(\region_0)\}$ from above into one region, and $\region_2$ as the other region.

Applying the idea to our setting (which is slightly different because we support $|\actions| > 2$), 
to split a region $\region_0 = \region(\prefix, \legalactions)$, 
we assume inductively that we have already sampled $G(\region_0)$ and $X(\region_0)$.
Let prefix $\prefix$ have $t$ states and $X(\region_0) = (a_0, \ldots, a_{t-1})$.
Note that $X(\region_0) \in \region_0$ by definition, so $\prefix$ is a prefix of $X(\region_0)$
and $a_{t} \in \legalactions$.
We can then define $\region_1 = \region(\prefix \oplus a_t, \actions)$
and $\region_2 = \region(\prefix, \legalactions \backslash \{a_t\})$.
We then need $G$ and $X$ for the new regions.
First, $X(\region_0) \in \region_1$, so it must be the case that it continues to be the 
argmax when considering a smaller region. Thus $\region_1$ ``inherits'' the parent's max and argmax: 
$G(\region_1) = G(\region_0)$ and $X(\region_1) = X(\region_0)$.
Creating a child region that does not contain the parent argmax follows the same logic as in standard
Top-Down sampling: $G(\region_2) \sim \TruncGumbel(\log \mu(\region_2), G(\region_0))$, and
we can sample $X(\region_2) \sim 1\{\trajectory \in \region_2\} \modelp{\trajectory \mid s_{\prefix}}$ as described
in the previous subsection.

\paragraph{Top-Down Sampling Trajectories.}

Adapting the search space structure from Kim et al.~makes it practical to implement
Top-Down sampling for trajectories. However, the algorithm is wasteful in its interactions
with the environment, particularly if trajectories can be long, because $X(\region)$
is instantiated fully for each region that is put on the queue.
This would also prevent applying the algorithm at all if trajectories are of infinite length.
We develop a further modification that addresses these issues.

Our idea is to use a similar search space as Kim et al.~but to lazily sample $X(\region)$.
The key observation is that the full value of $X(\region)$ is never used when splitting regions.
Paired with the fact that maxes and argmaxes are independent, this means that we are free to only
maintain prefixes of $X(\region)$ and sample extensions when they are needed.
Using the same notation as above, we just need samples of the next action $a_t$ to define the split.
In fact, we can do away with explicitly maintaining $X$'s in the algorithm altogether.
They can be recovered when we encounter a singleton region as the only trajectory in the region.
The resulting algorithm is our Modified Top-Down algorithm and appears in Algorithm \ref{alg:topdown}. 


\section{Additional Experimental Details}

\subsection{Combinatorial Bandits}
\label{app:combinatorial_bandits}
\ouralgshort~interacts with an environment to construct spanning trees as a sequence of binary decisions about whether to include each edge.
The environment provides a set of legal actions at each step.
If adding an edge would create a cycle, the only legal action is to not add the edge.
If there are $k$ steps left and only $n - k - 1$ edges so far, the only legal action is to add the edge.
If there is only one legal action, we take it with probability 1.
While this reduces the chance of the agent generating an invalid tree, 
it is possible to generate an invalid spanning tree,
in which case we continue searching over trajectories in descending order of $\Gtheta(\trajectory)$ until finding a valid tree.
The first valid tree found is returned as the agent's predicted tree.
The baseline methods always generate valid spanning trees.
Thus, this ensures that the algorithms are not being evaluated in terms of how quickly they learn to generate valid spanning trees. They are all evaluated in terms of how quickly they learn to generate spanning trees with high reward.

As baselines, we use a privileged "semi-bandit" version of UCB that observes per-edge rewards
and a version that assumes the per-tree rewards are attributed evenly to the edges, i.e., $r_e = \frac{r_{\mathcal{T}}}{n - 1}$.
Both baselines choose a tree at time $t$ by computing a maximum spanning tree given upper confidence bound edge costs $u_e = \hat \mu_e + \frac{1.5 \log t}{c_e}$
where $\hat \mu_e$ is the average per-edge reward for edge $e$ and $c_e$ is the number of times edge $e$ has been chosen. 

\subsection{DeepSea}

The policy model is a linear layer which gets as input one-hot vector of size 5x5 and outputs log probability for each action [FC(number of states, number of actions)]. We used Adam optimizer with a learning rate of 0.001 

\subsection{Minigrid}
\label{app:minigrid}

The observations are provided as a tensor of shape 7x7x3. Each of the $7 \times 7$ tiles is encoded using 3 integer values: one describing the type of object contained in the cell, one describing its color, and a flag indicating whether doors are open or closed. In addition, the agent's orientation is also provided as one-hot vector of size 4. 
 
 The policy model consists of 3 convolutional layers and one linear layer on top of them. $Conv1(3,32) \rightarrow ReLU \rightarrow Conv2(32,48) \rightarrow ReLU \rightarrow Conv3(48,64) $. The linear layer gets as input a concatenation of orientation vector and the output of the convolutional layers, namely $FC(64+4,7)$. The output of the linear layer is the log-probabilities of possible action.  We used Adam optimizer with a learning rate of 0.001. We used the same architecture for our algorithm and the baselines. 

We trained the model for 9M iterations, with a maximum of 3000 iterations per episode. In our algorithm we used the interactions budget for searching for direct candidates. In REINFORCE and cross-entropy method algorithms we used the interactions budget to sample 30 independent trajectories (100 steps trajectories) while we used the simulator to reset the environment.
For REINFORCE we averaged the gradients of the 30 trajectories before updating the policy model. For the cross-entropy method we averaged $\nabla_\theta \log \modelP{\trajectory \mid \staterewardtree}$ over the best $2$ out of $30$ trajectories. The results shown in \figref{fig:minigrid1} are an average of 5 trials with different random seeds.


We consider two versions of REINFORCE algorithm. The first is the standard trajectory-level  $\nabla\expect_{\trajectorytriple \sim p_{\theta}}\left[ \sum_{t=0}^{T-1} r_t \right] = \sum_{t=0}^{T-1}\nabla_\theta \log \modelp{a_t \mid s_t}\sum_{i=0}^{T-1}r_i$. However, the variance of the trajectory-level is high. The other version is an action-level which consider only the future rewards and serves as a variance reduction technique  $\nabla\expect_{\trajectorytriple \sim p_{\theta}}\left[ \sum_{t=0}^{T-1} r_t - b \right] = \sum_{t=0}^{T-1}\nabla_\theta \log \modelp{a_t \mid s_t}\sum_{i=t}^{T-1}r_i-b$ where the baseline $b$ is the average of the rewards over all time steps.


\end{document}